\documentclass{article} 
\usepackage{iclr2023_conference,times}

\usepackage{amsmath,amsfonts,bm}

\def\eqref#1{equation~\ref{#1}}

\def\1{\bm{1}}

\DeclareMathAlphabet{\mathsfit}{\encodingdefault}{\sfdefault}{m}{sl}
\SetMathAlphabet{\mathsfit}{bold}{\encodingdefault}{\sfdefault}{bx}{n}

\usepackage{url}
\usepackage{amsmath}
\usepackage{amssymb}
\usepackage{mathtools}
\usepackage{mathrsfs}
\usepackage{amsthm}
\usepackage{microtype}
\usepackage{graphicx}
\usepackage{subfigure,wrapfig}
\usepackage{multirow}
\usepackage{booktabs} 
\usepackage{tablefootnote}
\usepackage{enumitem}
\usepackage{wrapfig}
\usepackage[normalem]{ulem}
\usepackage[page,header]{appendix}
\usepackage{titletoc}
\usepackage{hyperref}
\usepackage{xspace}
\usepackage[nameinlink,capitalize]{cleveref}
\crefname{section}{\S}{\S}
\Crefname{section}{\S}{\S}
\crefname{appendix}{App.}{Apps.}
\Crefname{appendix}{App.}{Apps.}
\crefname{theorem}{Thm.}{Thms.}
\Crefname{theorem}{Thm.}{Thms.}
\crefname{proposition}{Prop.}{Props.}
\Crefname{proposition}{Prop.}{Props.}
\crefname{assumption}{Assumption}{Assumptions}
\Crefname{assumption}{Assumption}{Assumptions}
\crefname{algorithm}{Algorithm}{Algorithms}
\Crefname{algorithm}{Algorithm}{Algorithms}
\crefname{tabular}{Table}{Tables}
\Crefname{tabular}{Table}{Tables}
\newcommand{\dist}{FreD}

\theoremstyle{plain}
\newtheorem{theorem}{Theorem}[section]
\newtheorem{proposition}[theorem]{Proposition}

\theoremstyle{definition}
\newtheorem{definition}[theorem]{Definition}

\theoremstyle{remark}

\title{Privately Customizing Prefinetuning to Better Match User Data in Federated Learning}

\author{Charlie Hou \thanks{Work done while interning at Meta.} \\
Department of Electrical and Computer Engineering \\
Carnegie Mellon University \\
\texttt{charlieh@andrew.cmu.edu} \\
\AND
Hongyuan Zhan, Akshat Shrivastava, Sid Wang, Aleksandr Livshits \\
Meta \\
\texttt{\{hyzhan, akshats, yuwang2020, alll\}@meta.com}
\AND
Giulia Fanti \\
Department of Electrical and Computer Engineering \\
Carnegie Mellon University \\
\texttt{gfanti@andrew.cmu.edu} \\
\And
Daniel Lazar \\
Meta \\
\texttt{dlazar@meta.com} \\}

\iclrfinalcopy 
\begin{document}

\maketitle

\begin{abstract}
In Federated Learning (FL), accessing private client data incurs communication and privacy costs. 
As a result, FL deployments commonly \emph{prefinetune} \citep{aghajanyan2021muppet} pre-trained foundation models on a (large, possibly public) dataset that is held by the central server; they then \emph{FL-finetune} the model on a private, federated dataset held by clients \citep{nguyen2022begin}.  Evaluating prefinetuning dataset quality reliably and privately is therefore of high importance.  
To this end, we propose \dist~(Federated Private Fréchet Distance) --- a \textit{privately} computed distance between a prefinetuning dataset and federated datasets. 
Intuitively, it privately computes and compares a Fr\'{e}chet distance between embeddings generated by a large language model on both the central (public) dataset and the federated private client data. 
To make this computation privacy-preserving,  we use distributed, differentially-private mean and covariance estimators. 
We show empirically that \dist{} accurately predicts the best prefinetuning dataset at minimal privacy cost. 
Altogether, using \dist{} we demonstrate a proof-of-concept for a new approach in private FL training: (1) customize a prefinetuning dataset to better match user data (2) prefinetune (3) perform FL-finetuning.
\end{abstract}

\section{Introduction}

Federated Learning (FL) is a framework in which a central server learns a model from data that is distributed across a set of clients, without directly accessing that data \citep{mcmahan2017communication, kairouz2021advances}. 
One of the main motivations for FL is privacy: an early hope was that by not accessing client data directly, the central server would learn less about it, thereby protecting client privacy.
However, this intuition can be broken under naive implementations of FL \citep{carlini2021extracting, shokri2017membership}; to achieve meaningful privacy, one needs provably private training mechanisms, e.g., using differential privacy (DP) \cite{dwork2006differential,abadi2016deep}. 

Despite its privacy benefits, DP training of FL models incurs high utility costs. 
For example, in the widely-used DP stochastic gradient descent (DP-SGD) \cite{abadi2016deep}, to achieve reasonable privacy guarantees, models can only be trained for a limited number of rounds \cite{abadi2016deep}.

To get around this challenge \textit{under high privacy requirements where the number of FL training rounds is scarce}, a common approach is to ``prefinetune'' FL models \citep{nguyen2022begin}. That is, given a pretrained foundation model (e.g., BERT),  finetune it centrally on a dataset that is either public or owned by the FL coordinator, without privacy. The resulting \emph{prefinetuned} model
is used to initialize the federated model, which is sent to all clients and trained with private optimization. Prefinetuning helps the finetuning require fewer training steps, thereby boosting privacy guarantees. In this paper we use the term `FL-finetuning' to refer to finetuning on federated datasets.

While FL-finetuning is widely used today \citep{nguyen2022begin}, a crucial factor for its success is the choice of prefinetuning dataset.
For example, when training a large language model (LLM), one could prefinetune on a number of public datasets---e.g., Reddit \citep{caldas2018leaf} or StackOverflow \citep{reddi2020adaptive}. The efficacy of pre-training will ultimately depend on how closely the prefinetuning dataset represents the true, private data \citep{gu2022choosing,tramer2022considerations}.

Although prefinetuning dataset selection is critical to the success of FL finetuning, we lack algorithms to methodically select prefinetuning datasets, particularly in the FL setting (i.e., distributed private dataset under privacy constraints).

\begin{wrapfigure}{l}{0.55\textwidth}
    \centering
    \includegraphics[width=0.55\textwidth]{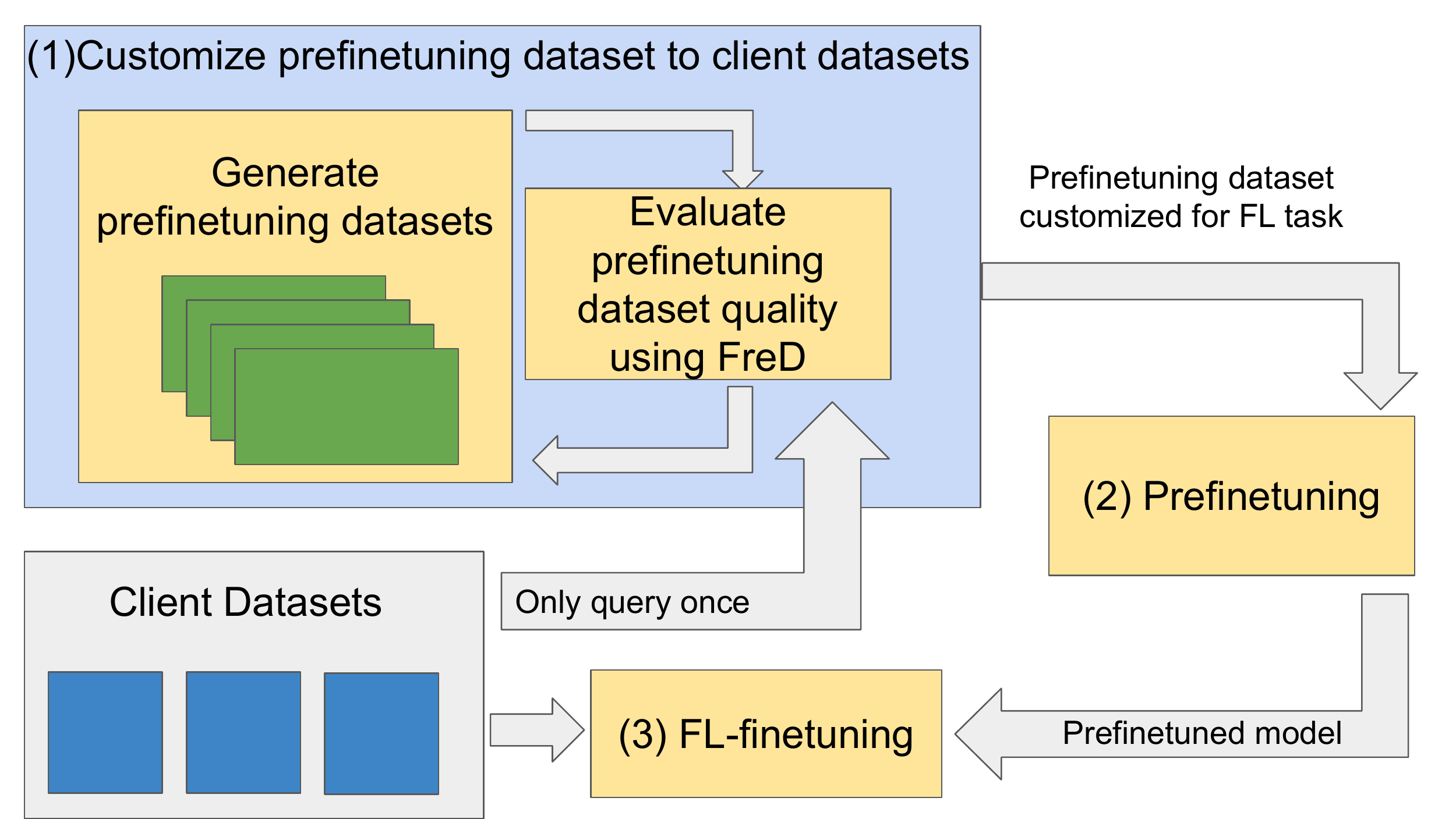}
    \caption{Proposed prefinetuning dataset customization approach using \dist{} to evaluate closeness of prefinetuning dataset to client data.  We demonstrate that it is possible to privately and repeatedly evaluate prefinetuing dataset quality using \dist{} in step (1), and demonstrate the end-to-end approach experimentally.}
    \label{fig:customization}
    \vspace{-0.5cm}
\end{wrapfigure}
Our work aims to fill this gap. 
Concurrent work \citep{gu2022choosing} tackled a similar problem in the centralized setting; this method first computes two low-dimensional subspaces from the per-sample gradients for a small batch of public and private data, respectively.
The dataset distance is measured as a projection metric between the two subspaces \cite{edelman1998geometry}. 
The algorithm in \citep{gu2022choosing} is designed for the centralized setting, and it is unclear how to apply it to a distributed dataset---particularly under the constraint that the distance measurement itself must be differentially private
\citep{papernot2021hyperparameter}. 

In this work, we instead modify a well-known dataset distance metric: the Fréchet Distance, most commonly used to measure dataset distances between synthetic generative adversarial network (GAN) and real datasets \citep{heusel2017gans} and occasionally used to measure distances between language datasets \citep{bertdistance}.  We demonstrate that a suitably chosen private formulation of the Fréchet Distance allows practitioners to accurately evaluate the quality of any number of prefinetuning datasets with respect to an FL task, with minimal privacy cost.  Therefore, given a dataset generation mechanism that can take feedback from our private Fréchet Distance (this can be as simple as choosing between already existing datasets, or could be as sophisticated as using an LLM to generate datasets) one can effectively customize their prefinetuning dataset for the FL task.  As a consequence, we make the first proof-of-concept for a new three-stage approach in privacy-sensitive FL applications \cref{fig:customization}: (1) customize the prefinetuning dataset to match client datasets (2) prefinetuning (3) FL-finetuning.

\paragraph{Contributions}
\begin{itemize}[leftmargin=*]
    \item We present \dist, a privacy-preserving metric based on Fréchet Distance to measure dataset distance between prefinetuning and FL-finetuning datasets.\footnote{Although \dist{} works well on non-federated settings, we focus on the federated setting in this work.} We show that \dist{} satisfies a formal $(\epsilon,\delta)$-differential privacy guarantee with respect to the private dataset. (\cref{thm:noisescale})
    \item \dist~computed with little privacy loss $(\epsilon = 0.6, \delta = 2 \times 10^{-6})$ empirically maintains enough resolution to accurately distinguish prefinetuning datasets that are only 1\% different. (\cref{fig:distmonotone})
    \item We show that a smaller \dist{} between the prefinetuning dataset and the FL-finetuning dataset leads to better FL-finetuning performance, measured in terms of test perplexity (Table \ref{tab:privatefinetune}).
    \item Taken together, we are the first to privately customize prefinetuning to better match user data for better FL model performance.  In our view, this is the main contribution of our work.

\end{itemize}


\section{Preliminaries}
We begin by defining differential privacy and introducing the Fréchet distance.  

\begin{definition}[Neighboring datasets]
    Two datasets $X, X'$ are said to be neighboring  (denoted $X\sim X'$) if they differ in a single record. Note that we consider a \emph{sample-level} notion of neighboring datasets; we allow one sample from one client to be added or removed.
\end{definition}

\begin{definition}[Differential Privacy]
A randomized algorithm $\mathcal{A}$ is $(\epsilon, \delta)$-differentially private (DP) if for any pair of neighboring datasets $X$, $X'$ and for all subsets $E$ of outputs, we have
\begin{align}
\text{Pr}[\mathcal{A}(X) \in E] \leq e^{\epsilon}~\text{Pr}[\mathcal{A}(X') \in E] + \delta.
\end{align}
\end{definition}
In this work, we will use a known DP mechanism called the Gaussian Mechanism \cite{dwork2006differential}, which adds Gaussian noise of a specific scale to protect privacy. 
To specify the scale, we must characterize the \emph{sensitivity} of the statistical query we wish to release. 
\begin{definition}[$\ell_2$ sensitivity \citep{dwork2014analyze}]
Let $g: X \to \mathbb{R}^p$ be a vector-valued function operating on datasets.  Let $X, X'$ be neighboring datasets.  The $\ell_2$-sensitivity of $g$
is defined as 
$\Delta g \triangleq \max_{X\sim X'}\|g(X) - g(X')\|_2$.
\end{definition}

When $X$ is a set of real-valued vectors (which is the case in our paper), bounding the $\ell_2$ sensitivity of some $g$ often requires an upper bound on the $\ell_2$ norm of the vectors in $X$. 
 The usual strategy to enforce this property for inputs to a DP mechanism is called \textit{clipping}:
 \begin{definition}[Clipping]
    We define the clipping operation $\chi_c : \mathbb{R}^d \to \mathbb{R}^d$ mapping a vector $v$ to a clipped version with $\ell_2$ norm at most $c$ to be
    $
        \chi_c(v) \triangleq v / \text{max}(1, \frac{\|v\|_2}{c}).
    $
    Suppose $E$ is some dataset of vectors of dimension $d$.  We will overload notation by letting $\chi_c(E)$ be the dataset $E$ where all its vectors have had the clipping operation $\chi_c$ performed on them.
 \end{definition}

\paragraph{DP Threat Model}  Our goal is to design DP algorithms in the high privacy regime (i.e. scarce FL training rounds) to defend against an adversary that accesses data at the central server.  
We assume the adversary has access to all intermediate quantities revealed to the central server.
Our algorithm for computing \dist{} therefore relies in part on secure aggregation \citep{bonawitz2017secureagg}, which allows the server to obtain a summary (e.g., sum) of the client data without access to individual client information. Because the central server only obtains a summary, 
the scale of noise required by the Gaussian Mechanism to achieve the same overall privacy is greatly reduced. 

\paragraph{Fréchet distance}
The Fréchet distance is a distance metric over probability distributions. 
For two probability measures $\eta$ and $\nu$ defined over $\mathbb R^n$, their Fréchet distance is defined as follows:
$$
d(\eta,\nu) \triangleq \left (\inf_{\gamma \in \Gamma(\eta,\nu)} \int ||x-y||^2 d\gamma(x,y) \right )^{1/2},
$$
where $\Gamma(\eta,\nu)$ is the set of all couplings of $\eta$ and $\nu$ (i.e. the set of all distributions $\gamma$ such that $\eta(x) = \int \gamma(x,y) dx$ and $\nu(y) = \int \gamma(x,y) dy$).
In the special case where $\eta$ and $\nu$ are Gaussian distributions $\mathcal N(\mu_1,\Sigma_1)$ and $\mathcal N(\mu_2,\Sigma_2)$, respectively, this can be written in closed form \citep{dowson1982frechet}: 
\begin{align}
    d(\eta,\nu) = ||\mu_1-\mu_2||^2_2 + \text{Tr} \left (\Sigma_1 + \Sigma_2 - 2 \left (\Sigma_1 \Sigma_2  \right )^{1/2} \right ).
    \label{eq:frechet}
\end{align}
Fréchet distance has been used in the GAN literature to evaluate the distance between synthetic and real datasets, using the Fréchet Inception Distance (FID) \cite{heusel2017gans}. 
The core idea of FID is to first extract representations of real and synthetic samples from the deepest hidden layer of a pre-trained Inception v3 model.
Then, treating those representations as multivariate Gaussians, one estimates the empirical mean and covariance of each set.
Finally, the distance between the two estimated distributions is computed using Fréchet distance (\eqref{eq:frechet}).  

While FID \citep{heusel2017gans} is used for images, it also maintains semantic closeness in the case of language.  \citet{bertdistance} use BERT as the embedder and show that a smaller Fréchet distance corresponds to human notions of language dataset closeness.  In this paper we use ALBERT, which is smaller and therefore better suited for FL. We adapt the Fréchet Distance to the space of training federated language models in \dist{}, where privacy constraints are critical. 



\begin{algorithm}[tb]
    \caption{\dist}
    \label{algo:distance}
    \begin{algorithmic}
        \STATE {\bfseries Input:}$X_1$, $X_2$ sentence datasets, where $X_1$ is the prefinetuning dataset (on server) and $X_2$ is the FL-finetuning dataset (distributed on clients), $f$ sentence embedder, $\mathcal{C}$ client set, $c$ embedding clipping norm
        \STATE \textbf{\underline{$\triangleright$ Compute mean and covariance of prefinetuning dataset}}
        \STATE Compute $E_1 = f(X_1)$, and $m_1 = \text{mean}(E_1)$, $C_1 = \text{Cov}(E_1)$
        \STATE \textbf{\underline{$\triangleright$ Compute DP mean of client datasets}}
        \STATE Let $\tau_1 = (2 c/n_2) (\sqrt{2 \log (1.25/\delta)}/{\epsilon})$
        \STATE \textbf{Server sends} $f$ to all clients, clients compute $E_2$
        \STATE Clients clip $E_2$ to get $\chi_c(E_2)$ and add $\mathcal{N}(0, \tau_1^2 I_{d \times d})$ to each sample in $\chi_c(E_2)$ for $\mathcal{M}_1(\chi_c(E_2))$
        \STATE \textbf{Server securely aggregates} mean of $\mathcal{M}_1(\chi_c(E_2))$ from the clients to get $\textsf{pmean}_{2}$.
        \STATE \textbf{\underline{$\triangleright$ Compute DP covariance of client datasets}}
        \STATE Let $\tau_2 = (c^2 / n_2) (\sqrt{2 \log (1.25/\delta)}/{\epsilon})$
        \STATE \textbf{Server sends} $\textsf{pmean}_{2}$, clients subtract $\textsf{pmean}_{2}$ from each sample in $\chi_c(E_2)$ for $\Theta(\chi_c(E_2))$
        \STATE Clients clip $\Theta(\chi_c(E_2))$ to get $\chi_c(\Theta(\chi_c(E_2)))$
        \STATE Clients compute their contributions to $\tilde{C}_2 := (1/n_2) \chi_c(\Theta(\chi_c(E_2)))^\top \chi_c(\Theta(\chi_c(E_2)))$ 
        \STATE Clients add $\mathcal{N}(0, \tau_2^2)$ independently to the upper triangle of their contributions of $\tilde{C}_2$, mirror the results to the lower triangle, and get $\mathcal{M}_2(\tilde{C}_2)$
        \STATE \textbf{Server securely aggregates} $\mathcal{M}_2(\tilde{C}_2)$ 
        \STATE Server projects $\mathcal{M}_2(\tilde{C}_2)$ to the nearest PSD matrix for $\textsf{pcov}_2$ 
        \STATE \textbf{\underline{$\triangleright$ Compute \dist{}}}
        \STATE Server computes $\|m_1 - \textsf{pmean}_{2}\|_2^2 + \text{Tr}(C_1 + \textsf{pcov}_2 - 2(C_1 \textsf{pcov}_2)^{1/2})$
    \end{algorithmic}
\end{algorithm}
\section{\dist: Method}
\label{sec:method}

Let $X_1$ and $X_2$ denote language datasets of $n_1$ and $n_2$ sentences, respectively.  
Let $f: S \to \mathbb{R}^{d}$ be a sentence embedder that maps from the space $S$ of sentences to a $d$-dimensional vector. We apply $f$ to each sentence in $X_1$ and $X_2$ to produce $E_1 \in \mathbb{R}^{n_1 \times d}$ and $E_2 \in \mathbb{R}^{n_2 \times d}$.  Then we calculate $m_1, m_2 \in \mathbb{R}^{d}$ to be the row-wise means of $E_1, E_2$ respectively and $C_1, C_2 \in \mathbb{R}^{d \times d}$ to be the row-wise covariances of $E_1$ and $E_2$ respectively. 


The Frechét distance is then computed as in \eqref{eq:frechet}:
\[
 d(X_1, X_2) = \|m_1 - m_2\|_2^2 + \text{Tr}(C_1 + C_2 - 2(C_1 C_2)^{1/2})
\]
Now, let $X_1$ be the prefinetuning dataset and $X_2$ be the FL-finetuning dataset. \cref{algo:distance} describes how we compute the private FL version of the Fréchet Distance, \dist.  At a high level, the algorithm's core is as follows: (1) we first calculate $\textsf{pmean}_2$ from the clients, which is the DP mean from $E_2$ (2) we send the DP mean to clients who then center their embeddings using the DP mean (the embeddings will not be exactly zero mean because the DP mean is not the true mean, but we find this is sufficient), and then calculate the DP covariance from these centered embeddings (3) we use the DP mean and DP covariance of the client data together with the mean and covariance of the prefinetuning dataset to get \dist{}.  We will now justify the scale of the noise we add in \cref{algo:distance}.   
\begin{proposition}
\label{thm:noisescale}
Let $\tau_1 = (2 c/n_2) (\sqrt{2 \log (1.25/\delta)}/{\epsilon})$ and $\tau_2 = (c^2 / n_2) (\sqrt{2 \log (1.25/\delta)}/{\epsilon})$ be the scale of the Gaussian noise added in \cref{algo:distance}.  Then calculating \dist{} as in \cref{algo:distance} satisfies $(2\epsilon, 2\delta)$-DP.
\end{proposition}
\begin{proof}
    From \cite{dwork2014analyze}[Theorem 2], given that the $\ell_2$ sensitivity of the mean of $\chi_c(E_2)$ is $2 c/n_2$, we know that the Gaussian mechanism with noise on the scale of $\tau_1$ maintains $(\epsilon, \delta)$-DP.

    Next, again observe from \cite{dwork2014analyze}[Theorem 2] that for a vector dataset $A$ (and its neighbor $A'$, vectors arranged row-wise), the $\ell_2$ sensitivity of $A^\top A'$ if each row has norm at most $1$ is 
    $
        \|A^\top A - A'^\top A'\|_2 \leq 1.
    $
    We can write $B := \chi_c(\Theta(\chi_c(E_2)))$ as $c A$ for some $A$.  Furthermore, we can also write $B'$ ($B'$ a neighbor of $B$) as $c A'$ where $A'$ is a neighbor of $A$.  Therefore,
    \begin{align}
        \|&B^\top B - B'^\top B'\|_2 = \|(c A)^\top (cA) - (c A')^\top (cA')\|_2 = c^2 \|A^\top A - A'^\top A'\|_2 \leq c^2.
    \end{align}
    Therefore, the Gaussian mechanism with noise on the scale of $\tau_2$ maintains $(\epsilon, \delta)$-DP for the released private covariance $\mathcal{M}_2(\tilde{C}_2)$. 
    By the sequential composition property of $(\epsilon, \delta)$-DP, releasing both the mean and the covariance in this way satisfies $(2\epsilon, 2\delta)$-DP.
\end{proof}

\section{Experiments}
\label{sec:results}



\label{sec:nonprivate-fred}
\begin{table}[t]
\begin{center}
\begin{tabular}{ |p{2cm}| p{2cm}||p{2cm}|p{3cm}|p{2cm}|  }
 \hline
Prefinetune data & FL-finetune data & Non-private \dist{} & $(0.6, 10^{-6})$-DP \dist & Perplexity Reddit-test\\
 \hline
 StackOverflow-train & Reddit-train   & 678.12  &  826.82 & 65.90\\
 \hline
 Wikitext-train & Reddit-train &  877.58 & 954.25 & 67.33 \\
 \hline
\end{tabular}
\end{center}
\caption{The closer dataset to Reddit (StackOverflow) leads to better FL-finetuning performance on Reddit. Moveover, we can identify that StackOverflow is the closer dataset to Reddit even when computing \dist~privately.}
\label{tab:choice}
\end{table}
\subsection{Choosing between two prefinetuning datasets}
In this subsection we study the case where a practitioner is choosing between two existing candidate prefinetuning datasets.
\paragraph{Experimental Setup}
In this subsection we use 3 datasets in our experiments: the StackOverflow language dataset \citep{reddi2020adaptive}, the Reddit language dataset derived from Reddit data released by \url{pushshift.io} \citep{caldas2018leaf}, and the Wikitext dataset \citep{merity2016pointer}.  Here we let StackOverflow-train and Wikitext-train be the possible choices for prefinetuning datasets and Reddit-train be the FL-finetuning dataset.  Performance is evaluated on Reddit-test. 
 All three are freely available open-source.  We use a DistilGPT-2 model \citep{sanh2019distilbert}, initialized with weights prefinetuned on various combinations of our public datasets.  The task is to use DistilGPT-2 to perform  next word prediction.  
The metric we use for this task is perplexity \citep{jelinek1977perplexity}.  
We train on the cross-entropy loss.  

\paragraph{Training Details}
In the prefinetuning stage, the batch size is 16 and we tuned the best learning rate using a Bayesian hyperparameter sweep over the range $[10^{-1}, 10e^{-6}]$ on the SGD optimizer.  We choose the representative prefinetuned model based on its performance on the validation set: i.e. if we train on StackOverflow-train, the choice is based on performance on StackOverflow-val.  We prefinetune for 10 epochs.  For the FL-finetuning stage, we select hyperparameters similarly given the prefinetuning initialization.  We perform the FL-finetuning non-privately.

\paragraph{Results}
In \cref{tab:choice}, the prefinetuning dataset closest to Reddit, StackOverflow, performs the best as the prefinetuning dataset.  Furthermore, when we calculate \dist{} under $(0.6, 10^{-6})$-DP, we can still easily identify which of Wikitext and StackOverflow is closer.  This experiment demonstrates end-to-end our proposed method of prefinetuning dataset customization, at all three steps \cref{fig:customization}.

\subsection{Choosing among a sequence of prefinetuning datasets}
The motivation of this setting is that, because calculating \dist{} \cref{algo:distance} only requires us to query the private user data once (to calculate $\textsf{pmean}_2$ and $\textsf{pcov}_2$), we can generate a sequence of datasets (the dataset generation can be as basic as finding existing public datasets and as sophisticated as generating synthetic data from LLMs) and evaluate their suitability for the FL task for no additional privacy cost after the first \dist{} calculation. Here, we generate a sequence of prefinetuning datasets, which are a mix of the StackOverflow and Wikitext datasets, and we evaluate FL-finetuning performance on a split of StackOverflow.  The goal here is to show that a highly private \dist{} metric (1) predicts the FL-finetuning performance with respect to prefinetuning dataset choice (2) can accurately tell apart prefinetuning datasets that are even very similar.

\paragraph{Experimental Setup}
In this subsection we use two datasets in our experiments: the StackOverflow language dataset \citep{reddi2020adaptive} and the Wikitext dataset \citep{merity2016pointer}.  We split StackOverflow-train into two datasets of equal size: StackOverflow-train1 and StackOverflow-train2. We let our choices of prefinetuning dataset be $Y$\% Stackoverflow-train1 and $100-Y$\% Wikitext-train, where $Y$ can vary between 0 and 100. 
 The overall dataset size is kept constant at 150k sentences.  When FL-finetuning, we FL-finetune on Stackoverflow-train2 (users and data are distributed).  We test on Stackoverflow-test before and after finetuning. We use a DistilGPT-2 model \citep{sanh2019distilbert}, initialized with weights prefinetuned on various combinations of our public datasets.  The task is to use DistilGPT-2 to perform  next word prediction.  The metric we use for this task is perplexity \citep{jelinek1977perplexity}.  We train to minimize the cross-entropy loss.  

\paragraph{Training Details}
In the prefinetuning stage, the batch size is 16 and used learning rate 0.002 with momentum 0.9, on the SGD optimizer.  We prefinetune for 50 epochs.  For FL-finetuning stage, we FL-finetune with noise scale of 1.5 and clipping 0.01, using the Opacus framework \citep{yousefpour2021opacus}.  We train for one epoch, and sample 100 clients per round.  The privacy cost incurred is $(\epsilon=1.26, \delta=10^{-6})$ from this stage.
\begin{table}[t]
\begin{center}
\begin{tabular}{||c c c||} 
 \hline
Y\% & Test perplexity before FL-finetune & Test perplexity after FL-finetune \\ [0.5ex]
 \hline\hline
 10  & 56.18  &  52.77 \\ 
 \hline
 40  & 44.01 & 42.47 \\
 \hline
 70  & 39.35 & 38.52 \\
 \hline
 95   & 37.18 & 36.59 \\  
 \hline  
 99   & 36.90 & 36.37 \\
 \hline
 100  & 36.88 & 36.33 \\ [1ex] 
 \hline
\end{tabular}
\end{center}
\caption{Test perplexity before and after DP FL-finetuning, as a function of Y\%, the percent of the prefinetuning dataset that is composed of StackOverflow-train1.  We observe that as Y\% increases, the better our model performs after prefinetuning (but before FL-finetuning) and also after FL-finetuning.  We also observe that the improvement from doing FL-finetuning on top of prefinetuning decreases as Y\% increases.}
\label{tab:privatefinetune}
\end{table}
\paragraph{Results}
First, in \cref{tab:privatefinetune}, we see that a closer dataset (as $Y$ increases, the closer our prefinetuning dataset) the better the test perplexity before and after finetune.  Furthermore, the gain we get from FL-finetuning over only prefinetuning decreases as $Y$ increases.  Next, we observe \cref{fig:distmonotone} that both non-private \dist{} and private $(\epsilon=0.6, \delta=2 \times 10^{-6})$ \dist{} corresponds strongly with $Y$, showing that even under strong privacy requirements, \dist{} still gives high-resolution information about the comparative closeness between two choices of prefinetuning datasets. In particular, in \cref{fig:distmonotone} (right), we see that evan for prefinetuning datasets that are only 1\% apart, it is possible to distinguish then using $(\epsilon=0.6, \delta=2 \times 10^{-6})$ \dist{} with high confidence.  This experiment shows another end-to-end example of our new proposed prefinetuning customization process \cref{fig:customization}, demonstrating that by using \dist{}, customization can be highly accurate.

\begin{figure}%
    \centering
   {\includegraphics[width=6cm]{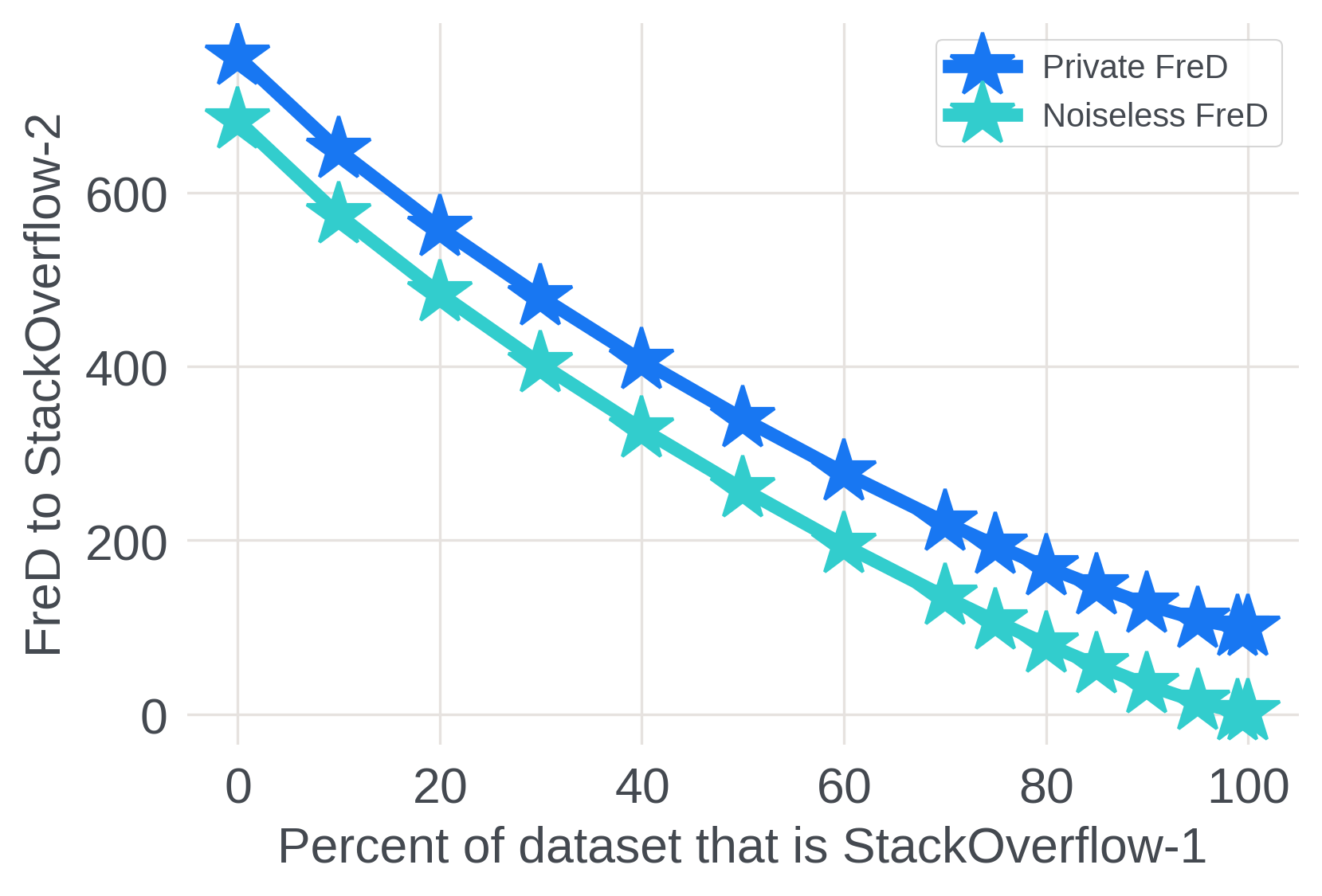} }%
    \qquad
    {\includegraphics[width=6cm]{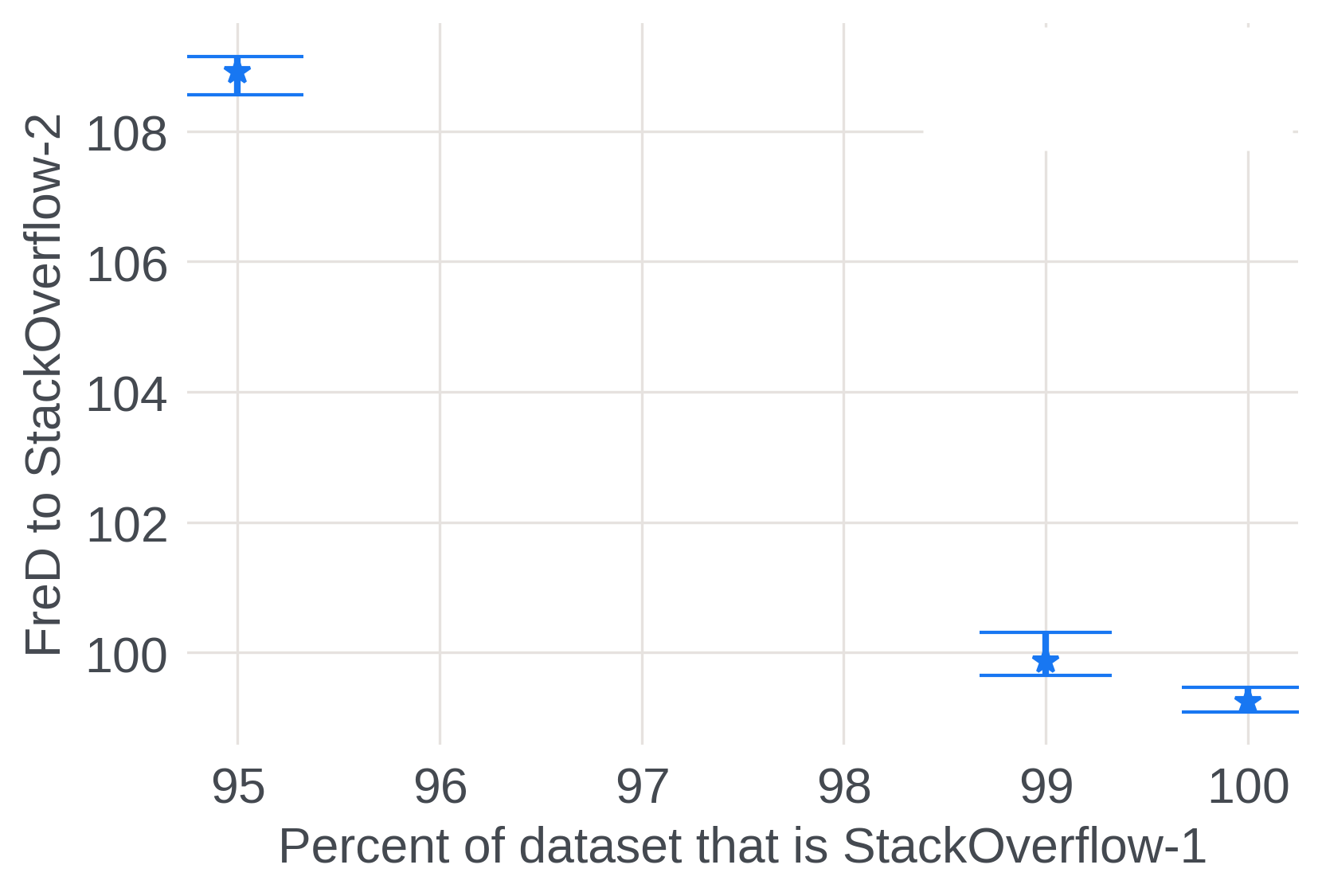} }%
    \caption{Left: We see that both private $(\epsilon=0.6, \delta=2\times 10^{-6})$ and non-private \dist{} monotonically decrease with the percentage of prefinetuning dataset that is StackOverflow-train1.  Right: We see that even between datasets that are 1\% apart, private \dist{} over 5 trials has nonoverlapping error bars--i.e. the highest observed \dist{} value for 100\% StackOverflow-train1 is still lower than the lowest observed \dist{} value for 99\% StackOverflow-train1.  This demonstrates that private \dist{} can distinguish between highly similar datasets with confidence.}%
    \label{fig:distmonotone}%
\end{figure}

\section{Conclusion}

In this paper, we make the case for using \dist{} as an informative metric for prefinetuning dataset choice.  Our experiments show that \dist{} is a good indicator of the quality of the prefinetuning dataset for an FL task.  By demonstrating this, we show that we can use \dist{} to privately customize our prefinetuning to match user data, which improves FL model performance.  Altogether, we demonstrate the first proof-of-concept of a new approach in privacy-sensitive FL applications: customization of the prefinetuning dataset for better FL model performance.  In the future, we plan to introduce more powerful prefinetuning dataset generation strategies to augment the power of our approach.  

\bibliography{iclr2023_conference}
\bibliographystyle{iclr2023_conference}

\appendix

\end{document}